\documentclass[12pt]{article}
\usepackage{amsthm}
\usepackage{amsmath, amssymb}
\usepackage{geometry}
\usepackage{hyperref}
\usepackage{fancyhdr}
\usepackage{titlesec}
\usepackage{algorithm}
\usepackage{algpseudocode}
\usepackage{tikz}
\usepackage{pgfplots}
\usepackage{booktabs} 

\titleformat{\section}{\large\bfseries}{\thesection}{1em}{}
\titleformat{\subsection}{\normalsize\bfseries}{\thesubsection}{1em}{}

\title{\textbf{Perturbation Analysis of Singular Values in Concatenated Matrices}}
\author{
    Maksym Shamrai\thanks{Institute of Mathematics of NAS of Ukraine, 3 Tereshchenkivska Str., 01601 Kyiv, Ukraine. E-mail: m.shamrai@imath.kiev.ua}
}
\date{}

\newtheorem{theorem}{Theorem}
\newtheorem{corollary}{Corollary}
\newtheorem{proposition}{Proposition}
\newtheorem{lemma}{Lemma}

\newtheorem{definition}{Definition}

\begin{document}

\maketitle

\begin{abstract}

Concatenating matrices is a common technique for uncovering shared structures in data through singular value decomposition (SVD) and low-rank approximations. The fundamental question arises: How does the singular value spectrum of the concatenated matrix relate to the spectra of its individual components? In the present work, we develop a perturbation technique that extends classical results such as Weyl's inequality to concatenated matrices. We setup analytical bounds that quantify stability of singular values under small perturbations in submatrices. The results demonstrate that if submatrices are close in a norm, dominant singular values of the concatenated matrix remain stable enabling controlled trade-offs between accuracy and compression. These provide a theoretical basis for improved matrix clustering and compression strategies with applications in the numerical linear algebra, signal processing, and data-driven modeling.
\end{abstract}

\section{Introduction}

We focus on the so-called \emph{concatenated matrices}, which are defined as those formed by a horizontal (or vertical) stacking of a collection of submatrices. In particular, given $N$ matrices $\{A_i\}_{i=1}^N$, each with the same number of rows, their horizontal concatenation is as follows
\[
A = [A_1 \; A_2 \; \cdots \; A_N].
\]
Considering a concatenated matrix makes it possible to jointly analyze multiple original matrices by extracting a common basis that is especially useful for revealing shared structures in different data representations.

The singular value decomposition (SVD) is a cornerstone of the numerical linear algebra. It serve as a tool for the dimensionality reduction, principal component analysis, and data compression \cite{Schmidt1907, EckartYoung1936}. In many modern applications, the data is naturally represented by multiple matrices -- each representing different aspects of an underlying phenomenon. A common strategy to reveal the intrinsic structure of the data consists of concatenating these matrices and extracting the shared basis through SVD. 

The problem of compressing large datasets arises in diverse applications~\cite{TseViswanath2005,GramfortEtAl2013,McMahanEtAl2017,KairouzEtAl2021}. Our study is addressing the question of how the number of submatrices affects the compression performance. When compressing $N$ matrices, one can either compress each matrix individually or concatenate a subset (or all) of the original matrices and then compress the resulting block matrix. By compressing the concatenated matrix, one can identify and exploit a common basis among the submatrices that should lead to a higher compression rate. On the other hand, compressing the matrices separately may result in a lower compression rate since the redundancy and shared information among the submatrices are not leveraged as effectively.

\paragraph{Problem statement.}
Given matrices $\{A_i\}_{i=1}^k\subset\mathbb{R}^{m\times n}$ and their
concatenation
\[
  M \;=\; [\,A_1,\,A_2,\,\dots , A_k\,]\in\mathbb{R}^{m\times kn},
\]
suppose that each block is perturbed to $\widetilde A_i=A_i+E_i$,
thus producing $\widetilde M=M+E$. \emph{How do the dominant singular values of $\widetilde M$ deviate from
  those of $M$ in terms of the individual block perturbations
  $\{E_i\}$ and the blocks number $k$?}

Classical perturbation bounds are based on Weyl’s inequality \cite{Weyl1912}
and its revisions by Davis–Kahan and Stewart–Sun~\cite{Davis1963,StewartSun1990, Stewart1998}.
They estimate the change in singular values of \emph{single matrix}
under global perturbations and reveal neither (i) how the blockwise errors accumulate inside concatenations, nor (ii) whether using $M^TM$ or $MM^T$ leads to a sharper bounds, usually employed in practical applications.

\paragraph{Notation.}
Throughout the paper we write $\|\cdot\|_2$ for the Euclidean
$\ell_2$-norm when the argument is a vector and for the spectral
(operator) norm when the argument is a matrix.

\vspace{0.5cm}

Main results of the present paper are as follows:
\begin{itemize}
    \item \textbf{Sharper estimate by $MM^T$.}
            In Section \ref{sec:bond} we compare $M^TM$ and $MM^T$ and show that
            using the latter removes quadratic cross‑terms in global
            bounds that leads to a provably tighter constant and provides a theoretical justification for the heuristic choice adopted in large-scale SVD implementations.
  \item \textbf{Block‑wise perturbation bound.}  
        In Section \ref{sec:main} we prove that
        \[
          |\sigma_i(\widetilde M)-\sigma_i(M)|
          \;\le\;
          \frac{1}{\sigma_i(M)}
          \sum_{j=1}^k \bigl(2\|A_j\|_2\,\|E_j\|_2 + \|E_j\|_2^2\bigr), \quad i = 1, \dots , \operatorname{rank}(M),
        \]
        whereas other singular values that appear beyond the rank obey
        
        \[
          \sigma_{i}(\widetilde M)\le
          \sqrt{\sum_{j=1}^k(2\|A_j\|_2\,\|E_j\|_2+\|E_j\|_2^2)}, \quad i = \operatorname{rank}(M) + 1, \dots, \min(m,kn).
        \]
        
        To our best knowledge, these are the \emph{first deterministic bounds for concatenated matrix that explicitly depend on the blocks number $k$}.
  
  \item \textbf{Compression–stability trade‑off.}
        Also, in Section \ref{sec:main} we show that the bounds translate into a quantitative guideline for deciding how
        many blocks to group before applying a low‑rank approximation:
        grouping is beneficial as long as the accumulated right‑hand side
        remains below a user‑specified tolerance.
\end{itemize}

\section{Preliminaries}

We begin with Weyl's inequality~\cite{Weyl1912}, which gives a uniform bound on the eigenvalues of a symmetric matrix under perturbations.

\begin{theorem}[\cite{Weyl1912}]
\label{t:1}
Let \(A, B \in \mathbb{R}^{n \times n}\) be symmetric matrices and \(E = B - A\) is a perturbation. Then for every eigenvalue (ordered in any fixed manner),
\[
|\lambda_i(B) - \lambda_i(A)| \leq \|E\|_2.
\]
\end{theorem}

Weyl's inequality assures that, if perturbation \(E\) is small in the spectral norm, then eigenvalues of the perturbed matrix \(B\) remain close to those of the original matrix \(A\). The property becomes important when we further examine how singular values are affected by perturbations in concatenated matrices.

The following proposition, along with its proof, is discussed in \cite{Horn_Johnson_1985}. It relates singular values of a matrix to the eigenvalues of its Gram matrix.

\begin{proposition}[\cite{Horn_Johnson_1985}]
\label{r:1}
For any matrix \(A \in \mathbb{R}^{n \times m}\) with rank \(r\), the singular values of \(A\) are given by
\[
\sigma_i(A) = \sqrt{\lambda_i(A^T A)}, \quad 1 \leq i \leq r.
\]
\end{proposition}

Consider now a matrix formed by the concatenation of the same copy of matrix $A$. The structure of the constructed matrix induces a certain regularity that we can exploit.

\begin{proposition}
\label{r:2}
Let \(A \in \mathbb{R}^{m \times n}\) be a matrix with the \(\operatorname{rank}(A)=r\). Define the concatenated matrix
\[
M \;=\; [\,A,\,A,\,\dots , A\,] \in \mathbb{R}^{m \times kn},
\]
which consists of \(k\) copies of \(A\). Define
\[
S_1 = M^T M \quad \text{and} \quad S_2 = M M^T.
\]
Then:
\begin{itemize}
    \item The nonzero eigenvalues of both \(S_1\) and \(S_2\) are as follows
    \[
    \lambda_i(S_1) = \lambda_i(S_2) = k\, \lambda_i(A^T A), \quad 1 \leq i \leq r.
    \]
    \item The remaining eigenvalues are the zero, i.e.,
    \[
    \lambda_i(S_1) = 0 \quad \text{for } r < i \leq kn,
    \]
    \[
    \lambda_i(S_2) = 0 \quad \text{for } r < i \leq m.
    \]
\end{itemize}
\end{proposition}

\begin{proof}
Note that
\[
S_2 = M M^T = [\,A,\,A,\,\dots , A\,]
\begin{bmatrix} A^T \\ A^T \\ \vdots \\ A^T \end{bmatrix} = k\, AA^T.
\]
If \(x_i\) is an eigenvector of \(A^T A\) corresponding to the eigenvalue \(\lambda_i(A^T A)\), then due to properties of the SVD, the nonzero eigenvalues of \(AA^T\) coincide with those of \(A^T A\). Consequently,
\[
S_2\, x_i = k\, AA^T\, x_i = k\, \lambda_i(A^T A) \, x_i, \quad i=1,\dots,r.
\]
Therefore, 
\[
\lambda_i(S_2) = k\, \lambda_i(A^T A), \quad i=1,\dots,r,
\]
with the remaining eigenvalues being zero. Since \(M^T M\) and \(MM^T\) share the same nonzero eigenvalues, the claim for \(S_1\) follows.
\end{proof}

The following proposition establishes a bound for a block matrix. This bound will be applied when analyzing block-structured perturbations.

\begin{proposition}[Bounding the Norm of a Block Matrix]
\label{r:3}
Let 
\[
E = \begin{bmatrix} 
E_{11} & E_{12} & \dots & E_{1k} \\ 
E_{21} & E_{22} & \dots & E_{2k} \\ 
\vdots & \vdots & \ddots & \vdots \\ 
E_{k1} & E_{k2} & \dots & E_{kk} 
\end{bmatrix}
\]
be a block matrix. Then,
\[
\|E\|_2 \le \sqrt{\sum_{i=1}^{k} \sum_{j=1}^{k} \|E_{ij}\|_2^2}.
\]
\end{proposition}

\begin{proof}
For any nonzero vector
\[
x = \begin{bmatrix} x_1 \\ x_2 \\ \vdots \\ x_k \end{bmatrix},
\]
we have
\[
Ex = \begin{bmatrix} 
\sum_{j=1}^k E_{1j} x_j \\[1mm]
\sum_{j=1}^k E_{2j} x_j \\[1mm]
\vdots \\[1mm]
\sum_{j=1}^k E_{kj} x_j 
\end{bmatrix}.
\]
Define \(y_i = \sum_{j=1}^{k} E_{ij} x_j\) for \(i=1,\dots,k\). The triangle inequality and submultiplicative property of the spectral norm deduce
\[
\|y_i\|_2 \le \sum_{j=1}^{k} \|E_{ij} x_j\|_2 \le \sum_{j=1}^{k} \|E_{ij}\|_2 \|x_j\|_2
\]
and, according to the Cauchy–Schwarz inequality,
\[
\|y_i\|_2^2 \le \left(\sum_{j=1}^{k} \|E_{ij}\|_2^2\right) \left(\sum_{j=1}^{k} \|x_j\|_2^2\right).
\]
Because \(\|x\|_2^2 = \sum_{j=1}^{k} \|x_j\|_2^2\), summation over \(i\) yields
\[
\|Ex\|_2^2 = \sum_{i=1}^{k} \|y_i\|_2^2 \le \left(\sum_{i=1}^{k} \sum_{j=1}^{k} \|E_{ij}\|_2^2\right) \|x\|_2^2,
\]
so that taking the square root and subsequently maximizing over all \(x \neq 0\) gives
\[
\|E\|_2 \le \sqrt{\sum_{i=1}^{k} \sum_{j=1}^{k} \|E_{ij}\|_2^2}.
\]
\end{proof}

\section{The Bound Estimation}
\label{sec:bond}

In this section, we establish an upper bound on the perturbation error for a concatenated matrix.

Let \( \{A_i\}_{i=1}^k \) be a collection of matrices with \( A_i \in \mathbb{R}^{m \times n} \) for each \( i = 1, \dots , k \). Define the original concatenated matrix and its perturbed one as
\[
M = [A_1, A_2, \dots, A_k] \quad \text{and} \quad \widetilde{M} = [\widetilde{A}_1, \widetilde{A}_2, \dots, \widetilde{A}_k],
\]
respectively, so that the perturbation is
\[
E = \widetilde{M} - M = [E_1, E_2, \dots, E_k],
\]
where \( E_i = \widetilde{A}_i - A_i \) for \( i = 1, \dots , k \).

The following lemmas characterize the perturbation structure.

\begin{lemma} \label{l:1}
The perturbation term \( D = \widetilde{M}^T \widetilde{M} - M^T M \) has the upper bound
\[
\|D\|_2 \leq \sqrt{ \sum_{i=1}^k \sum_{j=1}^k \big( \|A_i\|_2 \|E_j\|_2 + \|E_i\|_2 \|A_j\|_2 + \|E_i\|_2 \|E_j\|_2 \big)^2 }.
\]

\end{lemma}

\begin{proof}
Recalling \( \widetilde{A}_i = A_i + E_i \) implies
\[
\widetilde{M}^T \widetilde{M} = \begin{bmatrix} (A_1 + E_1)^T \\ \vdots \\ (A_k + E_k)^T \end{bmatrix} [(A_1 + E_1), \dots, (A_k + E_k)]
\]
so that expanding the product yields
\[
\widetilde{M}^T \widetilde{M} = \begin{bmatrix} (A_1 + E_1)^T (A_1 + E_1) & \dots & (A_1 + E_1)^T (A_k + E_k) \\ \vdots & \ddots & \vdots \\ (A_k + E_k)^T (A_1 + E_1) & \dots & (A_k + E_k)^T (A_k + E_k) \end{bmatrix},
\]
where each block takes the form
\[
(A_i + E_i)^T (A_j + E_j) = A_i^T A_j + A_i^T E_j + E_i^T A_j + E_i^T E_j.
\]
Therefore, the unperturbed term takes the form
\[
M^T M = \begin{bmatrix} A_1^T A_1 & \dots & A_1^T A_k \\ \vdots & \ddots & \vdots \\ A_k^T A_1 & \dots & A_k^T A_k \end{bmatrix},
\]
and the perturbation matrix \( D \) reads as
\[
D = \begin{bmatrix} A_1^T E_1 + E_1^T A_1 + E_1^T E_1 & \dots & A_1^T E_k + E_1^T A_k + E_1^T E_k \\ \vdots & \ddots & \vdots \\ A_k^T E_1 + E_k^T A_1 + E_k^T E_1 & \dots & A_k^T E_k + E_k^T A_k + E_k^T E_k \end{bmatrix} = 
\begin{bmatrix} D_{11} & \dots & D_{1k} \\ \vdots & \ddots & \vdots \\ D_{k1} & \dots & D_{kk} \end{bmatrix},
\]
where the individual block components are \( D_{ij} = A_i^T E_j + E_i^T A_j + E_i^T E_j \).

By applying the triangle inequality and the submultiplicative property of the norm,
\[
\|D_{ij}\|_2 \leq \|A_i^T E_j\|_2 + \|E_i^T A_j\|_2 + \|E_i^T E_j\|_2 \leq \|A_i\|_2 \|E_j\|_2 + \|E_i\|_2 \|A_j\|_2 + \|E_i\|_2 \|E_j\|_2.
\]
By Proposition \ref{r:3},
\[
\|D\|_2 \leq \sqrt{ \sum_{i=1}^k \sum_{j=1}^k \|D_{ij}\|_2^2 },
\]
leading to
\[
\|D\|_2 \leq \sqrt{ \sum_{i=1}^k \sum_{j=1}^k \big( \|A_i\|_2 \|E_j\|_2 + \|E_i\|_2 \|A_j\|_2 + \|E_i\|_2 \|E_j\|_2 \big)^2 }.
\]
\end{proof}

\begin{lemma}
\label{l:2}
The perturbation term \(D = \widetilde{M} \widetilde{M}^T - M M^T\) has the upper bound
\[
\|D\|_2 \leq \sum_{i=1}^k \left( 2 \|A_i\|_2 \|E_i\|_2 + \|E_i\|_2^2 \right).
\]

\end{lemma}

\begin{proof}
We begin with expressing \(\widetilde{M} \widetilde{M}^T\) based on the concatenated form of \(\widetilde{M}\), i.e.
\[
\widetilde{M} \widetilde{M}^T = [\widetilde{A}_1, \dots, \widetilde{A}_k] \begin{bmatrix} \widetilde{A}_1^T \\ \vdots \\ \widetilde{A}_k^T \end{bmatrix}
\]
that leads to the sum
\[
\widetilde{M} \widetilde{M}^T = \sum_{i=1}^k \widetilde{A}_i \widetilde{A}_i^T.
\]
Next, recall that \(\widetilde{A}_i = A_i + E_i\), thus, one can expand \(\widetilde{A}_i \widetilde{A}_i^T\) as follows
\[
\widetilde{A}_i \widetilde{A}_i^T = (A_i + E_i)(A_i + E_i)^T = A_i A_i^T + A_i E_i^T + E_i A_i^T + E_i E_i^T.
\]
As a consequence, we can write the sum as
\[
\widetilde{M} \widetilde{M}^T = \sum_{i=1}^k \left( A_i A_i^T + A_i E_i^T + E_i A_i^T + E_i E_i^T \right).
\]
The original matrix \(M M^T = \sum_{i=1}^k A_i A_i^T\), therefore, the perturbation term 
\[
D = \widetilde{M} \widetilde{M}^T - M M^T = \sum_{i=1}^k \left( A_i E_i^T + E_i A_i^T + E_i E_i^T \right).
\]
To estimate the bound \(\|D\|_2\), we consequently apply the triangle inequality and the submultiplicative property of the spectral norm:
\[
\|D\|_2 \leq \sum_{i=1}^k (\|A_i E_i^T\|_2 + \|E_i A_i^T\|_2 + \|E_i E_i^T\|_2) \leq \sum_{i=1}^k \left( 2 \|A_i\|_2 \|E_i\|_2 + \|E_i\|_2^2 \right).
\]
\end{proof}

Through Lemmas \ref{l:1} and \ref{l:2}, we derived explicit bounds for perturbations in the matrices \( M^T M \) and \( M M^T \), which serve as the key components in understanding the stability of these matrices under small changes. The results demonstrate that the perturbation in \( M M^T \) is less significant compared to \( M^T M \), as it involves only diagonal blocks in its perturbation structure. This distinction provides a tighter bound for \( M M^T \), offering stronger stability guarantees.

\section{Main Result}
\label{sec:main}

\begin{theorem}[Singular Value Perturbation of Concatenated Matrix]
\label{t:2}

Let $\{A_i\}_{i=1}^k$ be a collection of matrices with $A_i \in \mathbb{R}^{m \times n}$ for all $i = 1, \dots, k$. Define the original concatenated matrix by
\[
M = [A_1, \dots, A_k] \in \mathbb{R}^{m \times kn}, \quad r=\operatorname{rank}(M).
\]
Also, let $\widetilde{M} = [\widetilde{A}_1, \dots, \widetilde{A}_k]$ be a perturbed version of $M$, and define the perturbation matrix
\[
E = \widetilde{M} - M = [E_1, \dots, E_k], \quad \text{with } E_i = \widetilde{A}_i - A_i \text{ for each } i = 1, \dots, k.
\]
Then, the following perturbation bounds for the singular values hold true:
\begin{itemize}
    \item For $i = 1, \dots, r$ (corresponding to the nonzero singular values of $M$),
    \[
    |\sigma_i(\widetilde{M}) - \sigma_i(M)| \leq \frac{1}{\sigma_i(M)} \sum_{j=1}^k \Bigl(2 \|A_j\|_2 \|E_j\|_2 + \|E_j\|_2^2\Bigr).
    \]
    \item For $i = r+1, \dots, \min(m,kn)$ (corresponding to the zero singular values of $M$),
    \[
    \sigma_i(\widetilde{M}) \leq \sqrt{\sum_{j=1}^k \Bigl(2 \|A_j\|_2 \|E_j\|_2 + \|E_j\|_2^2\Bigr)}.
    \]
\end{itemize}
\end{theorem}

\begin{proof}
We start with Weyl's inequality applied to the two corresponding symmetric matrices $M^T M$ and $MM^T$, to derive bounds on perturbations of their eigenvalues, which are linked to the singular values of $M$ and $\widetilde{M}$.

Consider $M^T M$ and using the Lemma~\ref{l:1}, write down
\[
\|D_1\|_2 \leq \sqrt{\sum_{i=1}^k \sum_{j=1}^k \Bigl(\|A_i\|_2 \|E_j\|_2 + \|E_i\|_2 \|A_j\|_2 + \|E_i\|_2 \|E_j\|_2\Bigr)^2},
\]
where the perturbation term $D_1 = \widetilde{M}^T \widetilde{M} - M^T M$.

Applying Weyl's inequality to the eigenvalues of $M^T M$ and $\widetilde{M}^T \widetilde{M}$ leads to
\[
|\lambda_i(\widetilde{M}^T \widetilde{M}) - \lambda_i(M^T M)| \leq \|D_1\|_2.
\]

According to the Proposition~\ref{r:1}, the eigenvalues of $M^T M$ and $\widetilde{M}^T \widetilde{M}$ are squares of the singular values, i.e.,
\[
\lambda_i(\widetilde{M}^T \widetilde{M}) = \sigma_i^2(\widetilde{M}) \quad \text{and} \quad \lambda_i(M^T M) = \sigma_i^2(M) \, \Rightarrow \, |\sigma_i^2(\widetilde{M}) - \sigma_i^2(M)| \leq \|D_1\|_2.
\]

In similar way, according to Lemma~\ref{l:2}
\[
\|D_2\|_2 \leq \sum_{i=1}^k \Bigl(2 \|A_i\|_2 \|E_i\|_2 + \|E_i\|_2^2\Bigr),
\]
where $D_2 = \widetilde{M}\widetilde{M}^T - MM^T$.
Weyl's inequality implies
\[
|\lambda_i(\widetilde{M}\widetilde{M}^T) - \lambda_i(MM^T)| \leq \|D_2\|_2,
\]
and, again, due to Proposition~\ref{r:1}
\[
\lambda_i(\widetilde{M}\widetilde{M}^T) = \sigma_i^2(\widetilde{M}) \quad \text{and} \quad \lambda_i(MM^T) = \sigma_i^2(M) \, \Rightarrow \, |\sigma_i^2(\widetilde{M}) - \sigma_i^2(M)| \leq \|D_2\|_2.
\]

Because bounds provided by Lemma~\ref{l:2} (and hence $\|D_2\|_2$) are stronger than those of Lemma~\ref{l:1}, we continue the proof with $\|D_2\|_2$.

\paragraph{Deriving the Final Bound.}
For $\sigma_i(M) > 0, \, i=1,\dots, r$, note that
\[
\sigma_i^2(\widetilde{M}) - \sigma_i^2(M) = \bigl(\sigma_i(\widetilde{M}) - \sigma_i(M)\bigr) \bigl(\sigma_i(\widetilde{M}) + \sigma_i(M)\bigr).
\]

Since singular values are nonnegative and $\sigma_i(\widetilde{M}) + \sigma_i(M) > 0$,
\[
|\sigma_i(\widetilde{M}) - \sigma_i(M)| \leq \frac{\|D_2\|_2}{\sigma_i(\widetilde{M}) + \sigma_i(M)}.
\]

Moreover, because $\sigma_i(\widetilde{M}) + \sigma_i(M) > \sigma_i(M)$,
\[
|\sigma_i(\widetilde{M}) - \sigma_i(M)| \leq \frac{\|D_2\|_2}{\sigma_i(M)} \leq \frac{1}{\sigma_i(M)} \sum_{j=1}^k \Bigl(2 \|A_j\|_2 \|E_j\|_2 + \|E_j\|_2^2\Bigr).
\]

For indices $i = r+1, \dots, \min(m, kn)$, $\sigma_i(M) = 0$. In this case,
\[
\sigma_i^2(\widetilde{M}) \leq \|D_2\|_2,
\]
which implies
\[
\sigma_i(\widetilde{M}) \leq \sqrt{\|D_2\|_2} \leq \sqrt{\sum_{j=1}^k \Bigl(2 \|A_j\|_2 \|E_j\|_2 + \|E_j\|_2^2\Bigr)}.
\]
\end{proof}

Two important corollaries of Theorem~\ref{t:2} arise when the matrix $M$ is constructed by repeating a single matrix $A$, with perturbations affecting only a subset of the repetitions.

\begin{corollary}[Perturbation Around the Centroid]
\label{c:1}
Let $A \in \mathbb{R}^{m \times n}$ with rank $r = \operatorname{rank}(A)$. Consider perturbations $\{E_i\}_{i=2}^{k}$ with $E_i \in \mathbb{R}^{m \times n}$ for $i=2,\dots,k$ and define the concatenated matrix by
\[
M = [A, A, \dots, A] \in \mathbb{R}^{m \times kn},
\]
whereas the perturbed matrix is
\[
\widetilde{M} = [A, A+E_2, A+E_3, \dots, A+E_k] \in \mathbb{R}^{m \times kn}
\]
(here $E_1 \equiv 0$).

Then,
\[
|\sigma_i(\widetilde{M}) - \sqrt{k} \, \sigma_i(A)| \leq \frac{1}{\sqrt{k}\,\sigma_i(A)} \sum_{j=2}^{k} \Bigl(2 \|A\|_2 \|E_j\|_2 + \|E_j\|_2^2\Bigr), \quad i = 1, \dots, r,
\]
and
\[
\sigma_i(\widetilde{M}) \leq \sqrt{\sum_{j=2}^{k} \Bigl(2 \|A\|_2 \|E_j\|_2 + \|E_j\|_2^2\Bigr)}, \quad i = r+1, \dots, \min(m,kn).
\]
\end{corollary}

\begin{proof}
The result follows from Theorem~\ref{t:2}. Because $A_j = A$ for all $j$ and $E_1 \equiv 0$, the bound in Theorem~\ref{t:2} becomes
\[
|\sigma_i(\widetilde{M}) - \sigma_i(M)| \leq \frac{1}{\sigma_i(M)} \sum_{j=2}^{k} \Bigl(2 \|A\|_2 \|E_j\|_2 + \|E_j\|_2^2\Bigr).
\]
Furthermore, from the Propositions~\ref{r:1} and \ref{r:2} it follows that
\[
\sigma_i^2(M) = \lambda_i(M^T M) = k\,\lambda_i(A^T A) = k\,\sigma_i^2(A) \,\, \Rightarrow \,\, \sigma_i(M) = \sqrt{k}\,\sigma_i(A).
\]
Substituting this into the previous inequality yields the result.
\end{proof}

Immediate consequence of Corollary~\ref{c:1} is a continuity result. Namely, if each perturbation is small, then the singular values of the perturbed matrix converge to the singular values of the original (scaled) matrix.

\begin{corollary}[Continuity of Singular Values under Small Perturbations]
\label{c:2}
Suppose that $\|E_j\|_2 < \epsilon$ for all $j$. Then,
\[
\lim_{\epsilon \to 0} \sigma_i(\widetilde{M}) = \sqrt{k}\,\sigma_i(A) \quad i = 1, \dots , r,
\]
and
\[
\lim_{\epsilon \to 0} \sigma_i(\widetilde{M}) = 0 \quad i = r + 1, \dots , \min(m,kn).
\]
\end{corollary}

\begin{proof}
Corollary~\ref{c:1} leads to
\begin{align*}
|\sigma_i(\widetilde{M}) - \sqrt{k}\,\sigma_i(A)| &\leq \frac{1}{\sqrt{k}\,\sigma_i(A)} \sum_{j=2}^{k} \Bigl(2 \|A\|_2 \|E_j\|_2 + \|E_j\|_2^2\Bigr)\\[1mm]
&\le \frac{1}{\sqrt{k}\,\sigma_i(A)} \sum_{j=2}^{k} \Bigl(2 \|A\|_2 \epsilon + \epsilon^2\Bigr)\\[1mm]
&= \frac{(k-1)\epsilon}{\sqrt{k}\,\sigma_i(A)} \Bigl(2 \|A\|_2 + \epsilon\Bigr), \quad i = 1, \dots, r.
\end{align*}
Therefore, it follows that
\[
\lim_{\epsilon \to 0} |\sigma_i(\widetilde{M}) - \sqrt{k}\,\sigma_i(A)| = 0.
\]

For $i = r+1, \dots, \min(m,kn)$, Corollary~\ref{c:1} implies
\begin{align*}
\sigma_i(\widetilde{M}) &\leq \sqrt{\sum_{j=2}^{k} \Bigl(2 \|A\|_2 \|E_j\|_2 + \|E_j\|_2^2\Bigr)}
\le \sqrt{\sum_{j=2}^{k} \Bigl(2 \|A\|_2 \epsilon + \epsilon^2\Bigr)} \\
&= \sqrt{(k-1)\epsilon \Bigl(2 \|A\|_2 + \epsilon\Bigr)} \,\, \Rightarrow \,\, \lim_{\epsilon \to 0} \sigma_i(\widetilde{M}) = 0.
\end{align*}
\end{proof}

Modern sensing and communication systems often produce hundreds of matrix-valued observations that share a nearly identical column space. A natural strategy to exploit this redundancy is to \emph{concatenate} multiple blocks and store a \emph{single} joint rank $r$ singular value decomposition (SVD), rather than computing separate truncated SVDs for each block. This leads to a tradeoff:

\begin{quote}
\textbf{Storage efficiency.}  
Concatenating more blocks reduces the number of scalars required per stored singular vector.

\textbf{Spectral perturbation.}  
However, concatenating too many noisy blocks perturbs the leading singular values and singular vectors, potentially degrading downstream tasks that rely on them.
\end{quote}

The task consists of designing a \emph{computationally efficient rule} that determines \emph{in advance} how many blocks can be safely merged without exceeding a prescribed absolute error tolerance~$\tau$ for the top singular values. This is addressed in the third corollary.

\begin{definition}[Spectral budget]\label{def:spectral-budget}
Fix a target rank \(r\in\mathbb{N}\) and an absolute tolerance \(\tau>0\).
We call the pair \((r,\tau)\) the \emph{spectral budget} and say that a perturbed matrix \(\widetilde M\) satisfies the
\emph{\((r,\tau)\)-spectral budget} with respect to a reference matrix
\(M\) if
\[
  \bigl|\sigma_i(\widetilde M)-\sigma_i(M)\bigr|\;\le\;\tau,
  \qquad i=1,\dots,r.
\]
\end{definition}

\begin{corollary}[Maximum concatenation length under a \((r,\tau)\)-spectral budget]
\label{c:kmax}
Let $A_0\in\mathbb{R}^{m\times n}$ be fixed and consider
\[
    M_{1:k}       = \bigl[A_0,\dots,A_0\bigr]\in\mathbb{R}^{m\times kn},
    \qquad
    \widetilde M_{1:k}= \bigl[A_0,\;A_0+E_2,\dots,A_0+E_k\bigr]\in\mathbb{R}^{m\times kn}, \quad k\ge 1.
\]

Fix a \((r,\tau)\)-spectral budget, which should hold
\begin{equation}\label{eq:sv-budget}
    \bigl|\sigma_i(\widetilde M_{1:k})-\sigma_i(M_{1:k})\bigr|\le\tau,
    \quad i=1,\dots,r.
\end{equation}

Then the largest group size $k_{\max}$ for which \eqref{eq:sv-budget} can be guaranteed is
\begin{equation}\label{eq:kmax-def}
    k\le k_{\max}(\tau)
    :=\left(
        \frac{\tau\,\sigma_r(A_0)}
             {2\|A_0\|_2\,\bar\varepsilon(k)+\bar\varepsilon(k)^2}
      \right)^{\!2},
\end{equation}
where \(
  \bar\varepsilon(k):=\max_{2\le j\le k}\|E_j\|_2.
\)
\end{corollary}

\begin{proof}
Because $M_{1:k}$ repeats $A_0$ horizontally, from the Proposition~\ref{r:2}
\[
  \sigma_i(M_{1:k})=\sqrt{k}\,\sigma_i(A_0), \quad i=1,\dots , r.
\]
According to Corollary~\ref{c:1},
\[
  \bigl|\sigma_i(\widetilde M_{1:k})-\sqrt{k}\,\sigma_i(A_0)\bigr|
  \;\le\;
  \frac{1}{\sqrt{k}\,\sigma_i(A_0)}
  \sum_{j=2}^{k}\!\bigl(2\|A_0\|_2\|E_j\|_2+\|E_j\|_2^2\bigr),
  \quad i=1,\dots,r.
\]
The inequality $\|E_j\|_2\le\bar\varepsilon(k)$ implies
\[
  \sum_{j=2}^{k}\!\bigl(2\|A_0\|_2\|E_j\|_2+\|E_j\|_2^2\bigr)
  \le(k-1)\bigl(2\|A_0\|_2\bar\varepsilon(k)+\bar\varepsilon(k)^2\bigr)
  \le k\bigl(2\|A_0\|_2\bar\varepsilon(k)+\bar\varepsilon(k)^2\bigr).
\]
Hence, because $\sigma_i(A_0)\ge\sigma_r(A_0)$,
\begin{align*}
    \bigl|\sigma_i(\widetilde M_{1:k})-\sqrt{k}\,\sigma_i(A_0)\bigr|
  &\le
  \frac{1}{\sqrt{k}\,\sigma_i(A_0)}
    \sum_{j=2}^{k}\!\bigl(2\|A_0\|_2\|E_j\|_2+\|E_j\|_2^2\bigr) \\
  &\le
  \frac{\sqrt{k}\bigl(2\|A_0\|_2\bar\varepsilon(k)+\bar\varepsilon(k)^2\bigr)}
       {\sigma_r(A_0)}, \quad i=1,\dots , r.
\end{align*}
Requiring the right-hand side to be $\le\tau$ yields
\[
  \sqrt{k}\bigl(2\|A_0\|_2\bar\varepsilon(k)+\bar\varepsilon(k)^2\bigr)
  \le\tau\,\sigma_r(A_0).
\]
Squaring both sides and rearranging gives the claimed bound $k\le k_{\max}(\tau)$ of \eqref{eq:kmax-def}.  
Monotonicity of $\bar\varepsilon(k)$ implies monotonicity of $k_{\max}(\tau)$, that completes the proof.
\end{proof}

Because \(\bar\varepsilon(k)\) does not decrease by \(k\), the right‑hand
side of \eqref{eq:kmax-def} decreases as we append blocks. Therefore, the first violation of the inequality
\eqref{eq:kmax-def} gives the maximal group size which is \emph{guaranteed} to keep every singular value within the absolute tolerance~$\tau$ defined in~\eqref{eq:sv-budget}. This simple formula tells the practitioner how many channels to merge \emph{before} running a large‑scale SVD.

\section{Conclusion}

We extended classical perturbation theory to concatenated matrices by deriving singular value perturbation bounds that leverage the block structure.

The key contributions are as follows:
\begin{itemize}
    \item In Section \ref{sec:bond} we established upper bounds for the perturbation errors in both \(M^TM\) and \(MM^T\) of a concatenated matrix \(M\), demonstrating that the perturbation bound derived from \(MM^T\) is generally stronger.
    \item In Section \ref{sec:main} a comprehensive perturbation bound for the singular values of a concatenated matrix was obtained. This bound quantitatively relates the deviations in singular values to the individual block perturbations, providing clear criteria under which the stability of the singular spectrum is maintained. We also illustrated the practical implication of our theoretical results through a clustering scenario, deriving a formula that could be used in joint-compression algorithms to determine the number of matrices to concatenate before running SVD.
\end{itemize}

The explicit perturbation bounds developed in this paper not only enhance our theoretical toolkit for analyzing concatenated matrices but also provide practical guidelines for ensuring the stability of algorithms in diverse application domains.

\section*{Acknowledgments}
The authors confirm that there is no conflict of interest and acknowledges financial support by the Simons Foundation grant (SFI-PD-Ukraine-00014586, M.S.) and the
project 0125U000299 of the National Academy of Sciences of Ukraine. We also express our gratitude to the Armed Forces of Ukraine for their protection, which has made this research possible.

\bibliographystyle{plain}
\bibliography{references}

\begin{thebibliography}{10}

\bibitem{Davis1963}
Chandler Davis.
\newblock The rotation of eigenvectors by a perturbation.
\newblock {\em Journal of Mathematical Analysis and Applications (US)}, 6, 1963.

\bibitem{EckartYoung1936}
Carl Eckart and Gale Young.
\newblock The approximation of one matrix by another of lower rank.
\newblock {\em Psychometrika}, 1(3):211--218, 1936.

\bibitem{GramfortEtAl2013}
Alexandre Gramfort, Martin Luessi, Eric Larson, Denis~A. Engemann, Daniel Strohmeier, Christian Brodbeck, R.~Ted Hlushchuk, and Matti H\"am\"al\"ainen.
\newblock Meg and eeg data analysis with mne‐python.
\newblock {\em Frontiers in Neuroscience}, 7(267):1--13, 2013.
\newblock Provides practical examples of time–frequency matrices for individual electrodes.

\bibitem{Horn_Johnson_1985}
Roger~A. Horn and Charles~R. Johnson.
\newblock {\em Matrix Analysis}.
\newblock Cambridge University Press, 1985.

\bibitem{KairouzEtAl2021}
Peter Kairouz, H.~Brendan McMahan, Brendan Avent, Aurélien Bellet, Mehdi Bennis, Arjun~Nitin Bhagoji, and et~al.
\newblock Advances and open problems in federated learning.
\newblock {\em Foundations and Trends in Machine Learning}, 14(1--2):1--210, 2021.
\newblock Comprehensive survey; see Section 2.3 for gradient snapshot representations.

\bibitem{McMahanEtAl2017}
H.~Brendan McMahan, Eider Moore, Daniel Ramage, Seth Hampson, and Blaise~Aguera y~Arcas.
\newblock Communication‐efficient learning of deep networks from decentralized data.
\newblock In {\em Proc.\ 20th International Conference on Artificial Intelligence and Statistics (AISTATS)}, volume~54 of {\em Proceedings of Machine Learning Research}, pages 1273--1282, 2017.
\newblock Introduces \emph{federated averaging}; Section 3 describes gradient aggregation across clients.

\bibitem{Schmidt1907}
Erhard Schmidt.
\newblock Zur theorie der linearen und nichtlinearen integralgleichungen: I. teil: Entwicklung willk{\"u}rlicher funktionen nach systemen vorgeschriebener.
\newblock {\em Mathematische Annalen}, 63(4):433--476, 1907.

\bibitem{Stewart1998}
Gilbert~W Stewart.
\newblock Perturbation theory for the singular value decomposition.
\newblock {\em Digital Repository at the University of Maryland}, 1998.

\bibitem{StewartSun1990}
Gilbert~W Stewart and Ji-guang Sun.
\newblock {\em Matrix Perturbation Theory}.
\newblock Academic Press, San Diego, 1990.

\bibitem{TseViswanath2005}
David Tse and Pramod Viswanath.
\newblock {\em Fundamentals of Wireless Communication}.
\newblock Cambridge University Press, Cambridge, 2005.
\newblock See Chapter 7 for multi‐antenna channel estimation.

\bibitem{Weyl1912}
Hermann Weyl.
\newblock Das asymptotische verteilungsgesetz der eigenwerte linearer partieller differentialgleichungen (mit einer anwendung auf die theorie der hohlraumstrahlung).
\newblock {\em Mathematische Annalen}, 71:441--479, 1912.

\end{thebibliography}
\end{document}